\documentclass[11pt]{article}

\usepackage{amsmath,amssymb,amsthm}
\usepackage{graphicx}
\usepackage{booktabs}
\usepackage{caption}
\usepackage{subcaption}
\usepackage{geometry}
\usepackage{natbib}
\usepackage[hidelinks]{hyperref}

\newtheorem{definition}{Definition}
\newtheorem{theorem}{Theorem}
\newtheorem{remark}{Remark}

\title{Can Machines Really See Objects in Images? \\
\vspace{10pt}
\small{A Study Based on Syntactic Distance and Visual Self-Referential Instances}}

\author{
  Xingyu Peng\textsuperscript{1,}\footnote{Equal contribution.}\,\,,
  Junran Wu\textsuperscript{1,*},
  Yue Hou\textsuperscript{1},
  Zhongliang Qiao\textsuperscript{2},
  Jiaheng Liu\textsuperscript{3}, 
  Shangzhe Li\textsuperscript{4}, \\
  Jichang Zhao\textsuperscript{1},
  Wenjun Wu\textsuperscript{1},
  Xianglong Liu\textsuperscript{1},
  Yongxin Tong\textsuperscript{1},
  Li Dong\textsuperscript{1},
  Ke Xu\textsuperscript{1,}\footnote{Correspondence to: Ke Xu \textless kexu@buaa.edu.cn\textgreater .}\\[10pt]
  \textsuperscript{1}SKLCCSE, Beihang University \\
  \textsuperscript{2}Xiaoyu Brain \\
  \textsuperscript{3}Nanjing University \\ 
  \textsuperscript{4}Central University of Finance and Economics \\[6pt]
}

\date{}

\begin{document}
\maketitle

\begin{abstract}
\noindent Can a vision model truly see an object, or does it only fit surface-level visual cues? 
Following Wittgenstein's view that the limits of language are the limits of the world, we view a model's recognition ability as bounded by the descriptive system it has learned. 
In current vision models, this system is often realized through learned feature representations that exploit local statistical cues. We therefore ask whether a model can still classify correctly when such local cues provide no stable basis for distinction.
We formalize this question with syntactic distance, which measures class separability through the symmetry of the operations mapping one class to the other: positive distance exposes exploitable local features, whereas zero distance requires global semantics rather than local rules. 
We construct a visual self-referential task in maximum-variance binary noise: positive samples contain a closed square, while negative samples contain an otherwise identical square with one flipped boundary pixel. The two classes differ in global semantics but have zero syntactic distance, making local statistical shortcuts unreliable.
Experiments on ResNets and Vision Transformers reveal a consistent phase-transition phenomenon, with accuracy collapsing to random guessing once the image scale crosses a critical point and does not recover within the tested range. 
Larger training sets and models only delay this collapse, while globally attentive ViTs reach it earlier. 
These results reveal a structural capability boundary of current architectures on global-concept tasks, suggesting that general intelligence may require creating new language, not reusing an existing one.
\end{abstract}

\section{Introduction}
\label{sec:intro}

When a vision model classifies an object in an image as a cat, a basic question arises. Does the model truly grasp the concept of a cat, or does it merely match the target by learning local pixel patterns such as fur texture, color, and ear shape~\citep{geirhos2020shortcut,geirhos2018imagenet}? Although this question appears philosophical, it directly concerns the capability boundary of deep learning models, the nature of machine visual cognition, and a basic problem that computer vision still needs to clarify.

The observation of Wittgenstein in the Tractatus, that the limits of language are the limits of the world, provides a useful philosophical view for this question~\citep{wittgenstein1921logisch}. From this view, whether an object can be seen at the visual level is essentially equivalent to whether it can be effectively described at the linguistic level. Seeing is not only a passive sensory process. It is an active construction in which raw visual input is placed into a logical and cognitive structure that can be thought and stated. An object that cannot be described by language cannot become a visual object with cognitive meaning. Conversely, an object that cannot be captured by vision also lies outside the expressive boundary of language. Vision and language therefore share a cognitive boundary, and together they define the range of the world that can be known and understood.

Classification is largely a linguistic act. To place an object into a class is to describe it with a concept in language. Thus, seeing, describing, and classifying are tightly linked. In the forward direction, a human sees an object, describes it in language, and then classifies it according to that description. In the reverse direction, stable and correct classification of an object can serve as evidence that the model has an adequate descriptive language for that object, and therefore sees it in a meaningful sense. Vision is also a major force that expands cognitive boundaries. Humans continuously see things that have not been seen before, which creates the need to name and define new concepts. Language then expands, and the cognitive framework grows with it.

In deep learning, the core mechanism of widely used image recognition models is essentially a set of machine-specific language assigned by humans. A model learns feature rules and classification standards defined by humans, and uses them to interpret and classify visual inputs. Following Wittgenstein's view, if this machine language cannot effectively describe an object in an image, then the object has no existence within the machine's cognitive framework. In this sense, the machine does not truly see the object, and it cannot be said to understand it in an essential way.

\subsection{An Intuitive Analogy}
\label{sec:analogy}
An image recognition task can be viewed as a mapping from the syntactic features of an image to its semantic content. A sealed-box analogy makes this distinction concrete. The box corresponds to the surface syntax of the image, while the object inside the box corresponds to the deeper semantics. The task is to determine the state of the object inside. Suppose the box contains water, and one needs to decide whether it is liquid or solid. Opening the box is not necessary, since touching the surface of the box and measuring its temperature can be enough. This is because liquid water and ice have many physical differences that can be measured externally. Therefore, even if one correctly distinguishes the state, this does not prove that one has opened the box and seen the inside.

Now suppose the box contains a coin, and the task is to decide whether it faces heads or tails. External measurements such as temperature or weight cannot solve this task, because the two sides of the coin are symmetric with respect to all measurable external features. There is no stable external cue for the distinction. In this case, one must open the box and directly observe the coin to make the correct judgment. Thus, if an agent can correctly determine whether the coin is heads or tails, it must have truly seen the object inside the box. Otherwise, it has not genuinely seen.

We argue that many mainstream image classification tasks, such as cat versus dog classification, are similar to the water-versus-ice case. Different classes contain rich local statistical differences, and a model can obtain high classification accuracy by fitting these statistical features~\citep{brendel2019approximating,jo2017measuring}. Such a model behaves like a highly sensitive thermometer. It does not need to understand the concept of a cat itself, since surface syntactic features are enough for classification. The key question is therefore how to test whether a model can open the box and see, that is, whether it can understand the deeper semantics of an image. Our approach is to construct a visual coin, where two classes remain completely symmetric across all local statistical features but differ essentially in global semantics.

\begin{figure}[!ht]
\centering
\begin{subfigure}[b]{0.45\textwidth}
\centering
\includegraphics[width=0.65\textwidth]{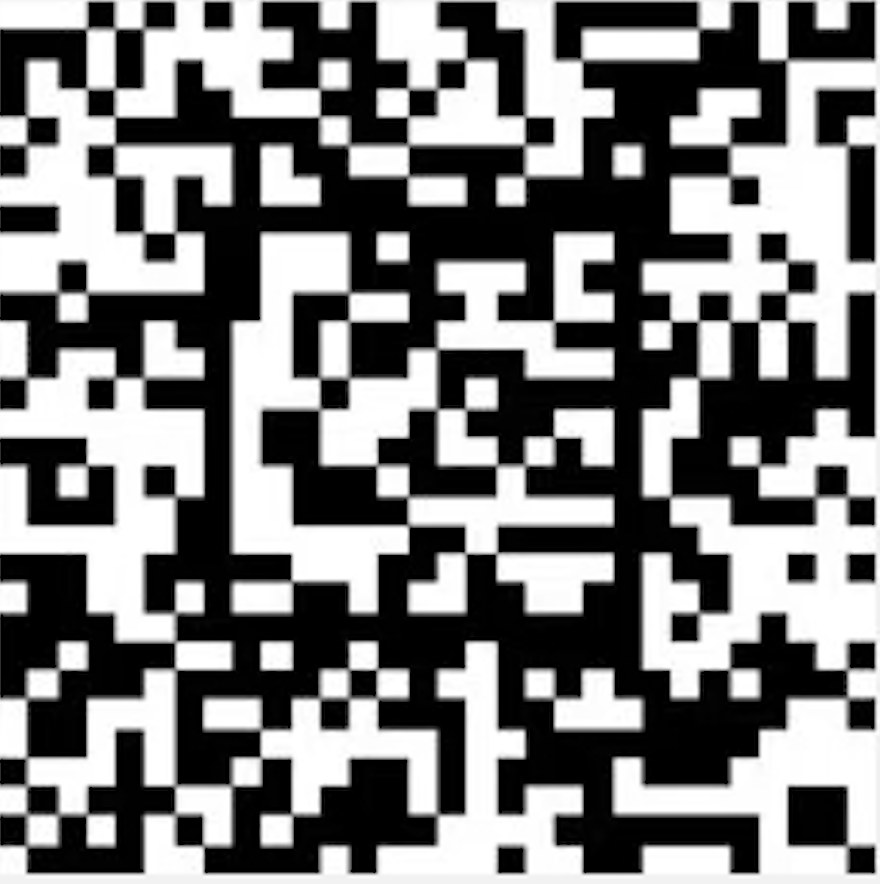}
\caption{Positive example: a closed square outline.}
\label{fig:square_full}
\end{subfigure}
\hfill
\begin{subfigure}[b]{0.45\textwidth}
\centering
\includegraphics[width=0.65\textwidth]{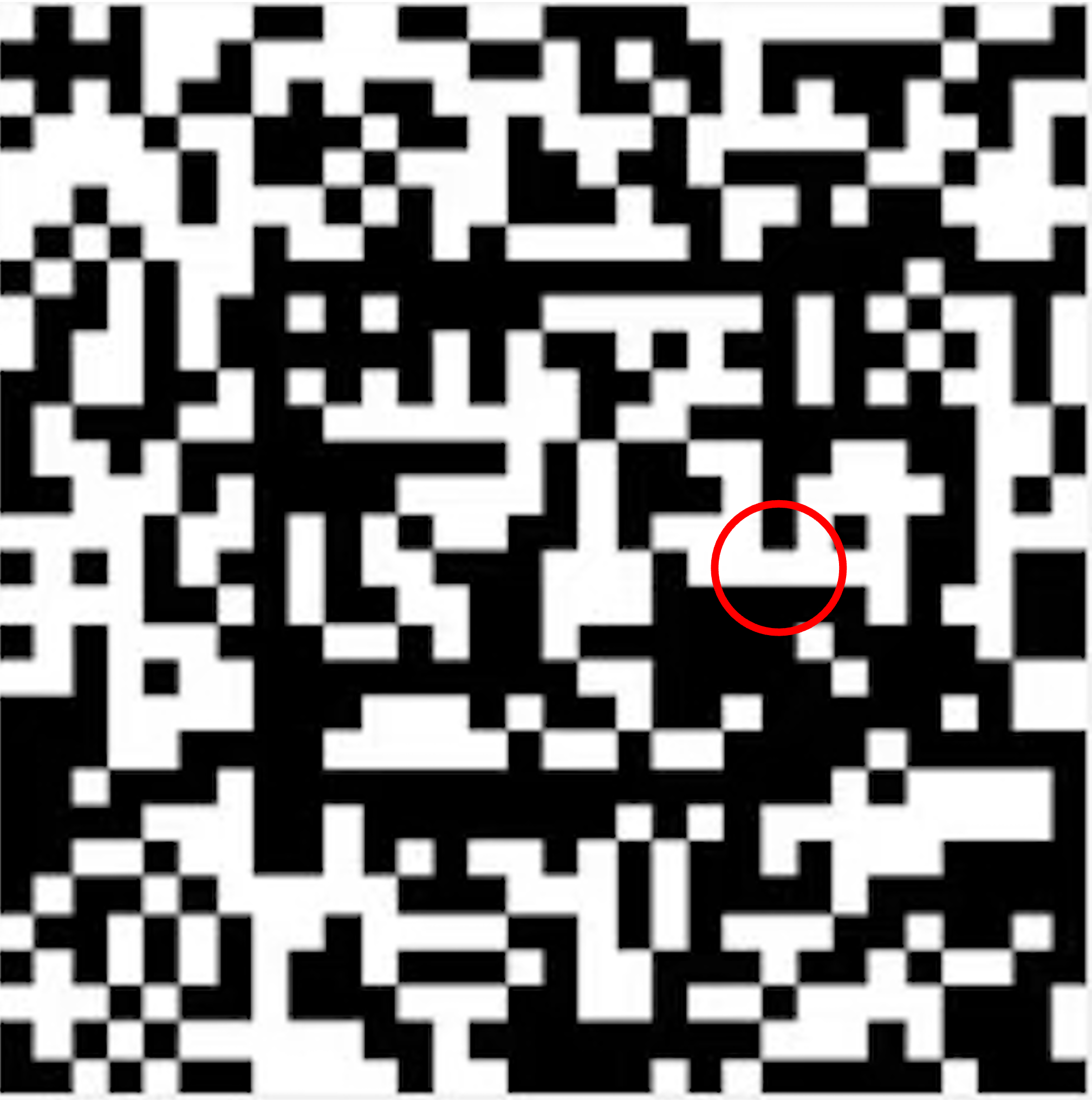}
\caption{Negative example: a broken square outline.}
\label{fig:square_miss}
\end{subfigure}
\caption{Examples from the visual self-referential dataset. Both classes contain a square outline embedded in a high-variance noise background. Positive examples have closed outlines, while negative examples have broken outlines, where only one pixel is flipped.}
\label{fig:dataset}
\end{figure}

\subsection{From Theory to Construction}
\label{sec:theory_to_construction}

To construct such a visual coin, we draw inspiration from self-referential instances in computational complexity theory. Recent studies on SAT, clique, and dominating set~\citep{xu2025sat,li2026constructing,zhou2026self} show that a small symmetry mapping can change a global property while keeping local structures highly similar. These results suggest that local shortcut reasoning cannot replace inspection of the global structure. At the philosophical level, this view echoes the core idea behind G\"odel's incompleteness theorem: truth may not be provable, and syntactic rules cannot substitute for semantic understanding.

We bring this theoretical view into computer vision. As shown in Figure~\ref{fig:dataset}, we embed a square outline into random black-and-white noise images and construct a binary classification task: distinguishing a closed square, whose boundary is complete, from a broken square, whose boundary contains a single flipped pixel. Because the background noise follows a maximum-variance random distribution, the effect of this one-pixel flip is weakened in global statistics. The two classes are hard to distinguish using standard local statistical cues. Therefore, to generalize on this task, a model needs to judge the global completeness of the square outline.

To characterize when objects can be distinguished by local features and when they require global judgment, we introduce \textbf{syntactic distance} as a formal tool. Syntactic distance quantifies feature distinguishability from the perspective of operational symmetry. A positive syntactic distance means that exploitable local feature differences exist, whereas zero syntactic distance means that rule-based local reasoning cannot provide a generalizable basis for distinction. In that case, distinguishing the two objects requires direct inspection of their global semantics. Under this framework, self-referential instances are instances with zero syntactic distance but different semantics, and the statement that local reasoning cannot replace global inspection corresponds to the consequence of zero syntactic distance.

\subsection{Main Findings}
\label{sec:findings}

Our experiments reveal a clear phase transition phenomenon. Analogous to the abrupt freezing of water at 0 degrees Celsius, the model's ability also changes sharply at a critical point. When the image size is small, the model can still rely on strong memorization capacity to fit the task~\citep{zhang2017understanding} and achieve high classification accuracy. As the image size increases, however, accuracy collapses at a critical point to around 50\%, which is the level of random guessing, and does not recover within the observed range.

This collapse appears in both CNNs, such as ResNets~\citep{he2016deep}, and Transformers with global attention, such as ViTs~\citep{dosovitskiy2021image}. This indicates that changing the architecture mainly shifts the critical point. In fact, ViTs with global attention collapse even earlier, further suggesting that architectural differences do not prevent the model from falling to random guessing within the present experimental range. These results indicate a clear capability bottleneck for current deep learning architectures when faced with global concept tasks in which local features fail to provide stable cues. In fact, this bottleneck resonates with the recently proposed ``Abstraction Barrier'' hypothesis~\citep{genewein2026agi}, which holds that AI systems trained on existing human abstractions and concepts lack the ability to discover novel concepts from raw data, and our work provides theoretical and empirical evidence for this hypothesis from the perspective of computer vision, through the syntactic distance framework and the phase-transition experiments on visual self-referential instances.

The starting point of intelligence is the creation of language. Creating language is not simply naming objects at the syntactic level. More importantly, it assigns semantics and content to them. Only through language can intelligence perform reasoning and solve complex problems. An intelligence that cannot create language is essentially imitation. Imitation can solve specific problems, but it cannot lead to general intelligence. Current large language models still reason within the existing human language system. They do not themselves create language and therefore remain limited on the path toward general intelligence.

\section{Syntactic Distance and Self-Referential Instances}
\label{sec:method}

To more clearly decouple local feature fitting from global concept understanding, we first introduce syntactic distance as a formal tool for characterizing when local reasoning fails. We then use this theory to construct a visual self-referential dataset, which allows local feature fitting and global concept understanding to be tested within the same classification setting.

\subsection{Theoretical Motivation: What Determines Whether Two Classes Can Be Seen Apart?}
\label{sec:motivation}

We return to the central question in the introduction. Why can some classification tasks be solved easily by local features, while others cannot? The key lies in whether the operations needed to turn one class of objects into another are symmetric.

Consider the coin example again. Turning heads into tails requires flipping the coin, and turning tails into heads also requires flipping the coin. The operation is the same in both directions. This symmetry makes it difficult to infer the internal state from external cues, and normally one must open the box and look directly. By contrast, turning ice into water requires heating, while turning water into ice requires cooling. The operations in the two directions are different. This asymmetry exposes a feature that can be measured externally, namely temperature, so one can make the judgment without opening the box.

The same logic applies to image classification. To edit an image of a dog into an image of a cat, the operation sequence may need to add striped patterns. In the reverse direction, it may need to remove striped patterns. Adding and removing are asymmetric operations. This asymmetry corresponds to discriminative texture differences between cats and dogs~\citep{geirhos2018imagenet,hermann2020origins}, which classifiers may exploit. This intuition can be formalized.

\subsection{Formal Definition of Syntactic Distance}
\label{sec:definition}

To rigorously describe whether two objects can be distinguished by local features under a given operation set, we define the following concepts.

\begin{definition}[Operation sequence pair]
\label{def:osp}
Given two objects $o_1, o_2$ and an operation set $\mathcal{A}$, an operation sequence pair $(f, g)$ consists of two operation sequences drawn from $\mathcal{A}$, such that
\begin{equation}
    o_1 = f(o_2), \qquad o_2 = g(o_1).
\end{equation}
Each operation sequence can be described in natural language or by a program, and therefore can be represented as a string.
\end{definition}

\begin{definition}[Syntactic distance]
\label{def:sd}
Given two objects $o_1, o_2$ and an operation set $\mathcal{A}$, let $\mathcal{P}$ be the set of all operation sequence pairs $(f, g)$ that transform $o_1$ and $o_2$ into each other. The syntactic distance is defined as
\begin{equation}
d_{\mathrm{syn}}(o_1, o_2 \mid \mathcal{A})
  \;=\;
  \min_{(f,\, g)\,\in\,\mathcal{P}} \; \mathrm{dist}\!\left(\, \mathrm{str}(f),\; \mathrm{str}(g) \,\right),
\end{equation}
where $\mathrm{str}(\cdot)$ denotes the string representation of an operation sequence, and $\mathrm{dist}(\cdot,\cdot)$ denotes a string distance, such as the Hamming distance. If no such operation sequence pair exists, that is, if $\mathcal{P} = \emptyset$, we set $d_{\mathrm{syn}} = \infty$.
\end{definition}

Intuitively, syntactic distance measures the minimum discrepancy between paired operation sequences that transform $o_1$ into $o_2$ and $o_2$ back into $o_1$.
If the paired operation sequences are identical, as in flipping a coin, the syntactic distance is zero, which means that the two objects are hard to distinguish at the operational level alone. If the paired operation sequences differ, as in heating versus cooling, the syntactic distance is positive, which means that operational asymmetry exposes discriminative features that can be exploited.

This definition has two central implications. When $d_{\mathrm{syn}} > 0$, the transformation operations in the two directions differ. This asymmetry exposes a structural difference that an algorithm can capture and use for discrimination, so local feature fitting can often provide an effective basis for classification. When $d_{\mathrm{syn}} = 0$, the operations in the two directions coincide. Local reasoning based on operational rules has difficulty distinguishing the two objects. To distinguish them, one needs to inspect the global meaning directly. In this case, global concept understanding becomes the key route to classification.

This leads to the central concept of this paper: self-referential instances are instances that have zero syntactic distance but different semantics. The condition $d_{\mathrm{syn}} = 0$ means that ordinary local shortcuts cannot replace a direct inspection of global semantics.

\begin{remark}
In practical computation, each operation in $\mathcal{A}$ can be assigned a fixed-length code, such as a binary code, and an operation sequence can be represented as a code string. Then the string distance does not depend on the particular linguistic description, but only on the inherent structure of operations.
\end{remark}

\subsection{Strings and Hamming Distance}
\label{sec:hamming}

Syntactic distance is a general framework. The following theorem shows that, under specific conditions, it reduces to a familiar classical distance measure, the Hamming distance. This provides a directly verifiable anchor for the definition of syntactic distance.

\begin{theorem}
\label{thm:hamming}
Let $o_1, o_2$ be two strings of equal length. Suppose the operation set $\mathcal{A}$ contains only operations that replace the current character at a specified position with a specified character, and suppose the string distance $\mathrm{dist}$ is the Hamming distance. Then
\begin{equation}
d_{\mathrm{syn}}(o_1, o_2 \mid \mathcal{A}) = d_H(o_1, o_2),
\end{equation}
where $d_H$ denotes the Hamming distance, namely the number of positions at which the two strings have different characters.
\end{theorem}

\begin{proof}
Let $o_1 = a_1 a_2 \cdots a_n$ and $o_2 = b_1 b_2 \cdots b_n$.

We first construct an operation sequence pair. For any operation sequence pair $(f, g)$, $f$ turns $o_2$ into $o_1$, and $g$ turns $o_1$ into $o_2$. Since each operation can replace only one character at one position, completing the transformation requires $f$ and $g$ to perform replacements at all positions where $o_1$ and $o_2$ differ. Therefore, we construct a shortest operation sequence pair as follows. At every different position $i$, where $a_i \neq b_i$, $f$ replaces the $i$-th character with $a_i$, while $g$ replaces the same position with $b_i$.

We then compute the string distance. After writing the two operation sequences as strings, the operations at position $i$ are replace with $a_i$ and replace with $b_i$. If $a_i \neq b_i$, the two operation codes differ. If $a_i = b_i$, no operation is needed at that position, and the two sequences agree. Therefore, $\mathrm{dist}(\mathrm{str}(f), \mathrm{str}(g))$ equals exactly the number of positions at which $a_i \neq b_i$, namely $d_H(o_1, o_2)$.

Finally, optimality follows because any operation sequence pair $(f', g')$ must cover all different positions, and at each such position the operations in $f'$ and $g'$ must differ, since they must set the position to $a_i$ and $b_i$ respectively, with $a_i \neq b_i$. Thus no operation sequence pair can have a smaller string distance.
\end{proof}

For example, let $o_1 = \texttt{abc}$ and $o_2 = \texttt{xyz}$. The sequence $f$ replaces the three positions with $\texttt{a}$, $\texttt{b}$, and $\texttt{c}$, while $g$ replaces the three positions with $\texttt{x}$, $\texttt{y}$, and $\texttt{z}$. The operations differ at all three positions, so $d_{\mathrm{syn}} = d_H = 3$.

This theorem shows that Hamming distance is a special case of syntactic distance when the objects are strings and the operations are restricted to single-character replacement. The generality of syntactic distance lies in the fact that it can handle arbitrary object types and arbitrary operation sets, rather than being limited to string editing.

\subsection{Classification Tasks from the View of Syntactic Distance}
\label{sec:classification}

With syntactic distance as a formal tool, we can understand the essential differences among classification tasks in a unified way and explain why existing benchmarks may fail to expose deficiencies in concept understanding.

\paragraph{Ice and water, $d_{\mathrm{syn}} > 0$.}
Let $\mathcal{A} = \{\text{heating}, \text{cooling}\}$. Turning ice into water requires heating, while turning water into ice requires cooling. The two operations differ, so $d_{\mathrm{syn}} > 0$. This operational asymmetry exposes temperature as an externally measurable discriminative feature.

\paragraph{Images of cats and dogs, $d_{\mathrm{syn}} > 0$.}
Let the operation set consist of image editing operations, such as texture modification and shape transformation. Turning an image of a dog into an image of a cat may require adding patterns, while turning a cat into a dog may require removing patterns. The operations differ, so $d_{\mathrm{syn}} > 0$. The operational asymmetry corresponds to discriminative texture features, which are precisely the features that deep learning classifiers may exploit~\citep{geirhos2018imagenet}. In such tasks, local feature fitting can often provide an effective basis for classification, and a model need not understand the global concept of cat to obtain the correct answer.

\paragraph{Heads and tails of a coin, $d_{\mathrm{syn}} = 0$.}
Let $\mathcal{A} = \{\text{flip}\}$. Turning heads into tails requires flipping, and turning tails into heads also requires flipping. The operations are identical in the two directions, so $d_{\mathrm{syn}} = 0$. Although heads and tails differ clearly in meaning, flipping is a self-inverse operation, and the symmetry of the operation makes external measurements unable to distinguish the two states. One usually has to open the box and observe directly. This is a typical self-referential instance: local reasoning provides no reliable basis, and direct inspection of the global state is required.

This analysis reveals a clear hierarchy. Many existing image classification benchmarks, such as ImageNet~\citep{deng2009imagenet}, fall into the regime $d_{\mathrm{syn}} > 0$. Categories contain rich asymmetric feature differences, and local feature fitting can already provide the main basis for classification. A model can achieve high accuracy by fitting these differences. This explains why high accuracy on such benchmarks does not fully prove that the model has mastered global concepts~\citep{baker2018deep,brendel2019approximating}. By contrast, our visual self-referential task falls into the regime $d_{\mathrm{syn}} = 0$. Local statistical features do not provide stable and generalizable discriminative cues. Global concept understanding becomes the key route to classification, and the model faces a challenge closer to the conceptual level.

\subsection{Construction of Visual Self-Referential Instances}
\label{sec:dataset}

Now, we construct a visual classification dataset that satisfies $d_{\mathrm{syn}} = 0$. We call its samples visual self-referential instances. They instantiate the concept of self-referential instances in the continuous visual domain: the two image classes are semantically different, namely complete square versus broken square, but they are difficult to distinguish by local operational rules in the sense of syntactic distance, with $d_{\mathrm{syn}} = 0$. The generation process is as follows.

\paragraph{High-variance background generation.}
We define the image canvas as an $N \times N$ black-and-white pixel matrix, where each pixel takes value 0 for black or 1 for white. All background pixels are independently sampled as black or white with probability 50\%, that is, from a Bernoulli distribution with $p = 0.5$. This probability setting is crucial. A 50/50 random distribution produces the theoretically largest statistical fluctuation, with variance $\sigma^2 = 0.25$. Small changes in the number of pixels are strongly weakened by natural noise, which reduces the possibility that an algorithm can classify using statistical shortcuts such as the ratio of black to white pixels.

\paragraph{Positive example generation, closed square.}
A square outline is generated at the center of the image. The square consists only of its boundary and is not filled. Its side length is fixed as $s = N/2$. All pixels on the continuous one-pixel-wide boundary are assigned the same value, either all black or all white.

\paragraph{Negative example generation, broken square.}
The negative examples are generated independently from the positive examples, meaning that the background noise is resampled, so that the two classes have the same background noise distribution. We generate a square outline of the same size at the center, again all black or all white. During generation, one pixel is uniformly sampled from the boundary and flipped, from black to white or from white to black, creating a tiny gap.

\paragraph{Key design considerations.}
The construction removes localization as a confounder because the square position and size are fixed at the image center. If a model fails, the failure is less likely to be caused by an inability to find the square and more likely to stem from difficulty in deciding whether the boundary is complete. The construction also makes local features isomorphic. Since the background noise follows a 50/50 random distribution, the forced flip of one pixel is statistically weakened by large background fluctuations. For ordinary local observation windows, or receptive fields, the local statistics near the boundary cannot be reliably distinguished between positive and negative examples. Finally, the syntactic distance is zero. As discussed in Section~\ref{sec:classification}, the transformation between positive and negative examples is a symmetric self-inverse operation. Thus local shortcut reasoning cannot reliably replace a direct inspection of the global completeness of the square.

The theoretical basis for this construction is solution independence. Recent work~\citep{zhou2026solution,zhou2026self} identifies solution independence and the resulting irreducibility as key conditions for constructing self-referential instances: when the value of a local component cannot be determined from the remaining components, local inspection cannot form a stable basis for the global decision. Our visual construction satisfies this condition. The maximum-variance 50/50 random background ensures that pixel values are mutually independent, so the value of a single pixel is hard to determine from the remaining pixels in the image. This pixel-level independence makes the flipped gap pixel statistically hard to leave a stable trace that can be detected locally, thereby enabling the construction of visual self-referential instances.

\paragraph{Correspondence to the box analogy.}
Returning to the sealed-box analogy in Section~\ref{sec:analogy}, the components of our visual self-referential instances correspond to the analogy as follows. The box corresponds to the surface features locally observable by the model, such as the random noise texture of the background and the lines and right angles of the square outline. These are syntactic signals that the model can read without opening the box. The state of the object inside the box corresponds to the semantic judgment of whether the square is complete, which is hard to infer stably from local features. Opening the box corresponds to the model moving from local feature fitting to global concept understanding, namely no longer relying on statistical shortcuts but recognizing that the object is a complete square. Since $d_{\mathrm{syn}} = 0$, the outside of the box, namely the local syntactic features, is highly similar for positive and negative examples, just as the outside measurements of a box containing a coin provide no stable difference between heads and tails. Correct classification can therefore be viewed as an operational test of whether the model can open the box and see, which is exactly the question posed by the title of this paper: can machines really see?

\subsection{Task Definition: Decoupling Feature Fitting from Concept Understanding}
\label{sec:tasks}

Based on the above dataset, we define two complementary binary classification tasks to separate the roles of local feature fitting and global concept understanding.

\paragraph{Task A: feature-driven classification, a control task with $d_{\mathrm{syn}} > 0$.}
The model must distinguish pure random noise images from images containing a closed square. Since the continuous boundary of the closed square introduces strong local features, such as lines and right angles, it differs sharply from pure noise. The syntactic distance between the two image classes is much larger than zero. This task evaluates low-level feature extraction. Success on Task A shows that the model can stably extract low-level visual features such as square boundaries, and it rules out the explanation that failure comes from a basic inability to detect geometric shapes.

\paragraph{Task B: concept-driven classification, the core self-referential task with $d_{\mathrm{syn}} = 0$.}
The model must distinguish closed squares, whose boundary is complete, from broken squares, whose boundary has a one-pixel gap. Because the syntactic distance between the two image classes is $d_{\mathrm{syn}} = 0$, they share the same background noise distribution and local edge features. Feature fitting alone is far from sufficient. Under the constraint $d_{\mathrm{syn}} = 0$, the model needs to recognize whether the square is complete in order to solve Task B. By testing models on gradually increasing image scales $N$, we characterize the capability boundary of feature-fitting-based deep learning on global concept tasks.

\section{Experiments}
\label{sec:experiments}
\subsection{Experimental Setup}
\label{sec:setup}

\paragraph{Dataset.}
We generate visual self-referential instances following Section~\ref{sec:dataset}. Let $N_{\mathrm{train}}$, $N_{\mathrm{val}}$, and $N_{\mathrm{test}}$ denote the number of images per class in the training, validation, and test sets, respectively. We set $N_{\mathrm{train}} \in \{5000, 10000, 20000, 40000\}$, and fix the validation and test sizes to $N_{\mathrm{val}}=2000$ and $N_{\mathrm{test}}=5000$ per class. For each image scale $N$, the training set is independently resampled for each value of $N_{\mathrm{train}}$, whereas the validation and test sets are kept fixed across all choices of $N_{\mathrm{train}}$. This design isolates the effect of the training-set size from sampling noise in the evaluation sets. We sweep the image scale $N$ from $20$ to $220$. To accurately locate the collapse transition points, we use dense sampling with step size $\Delta N = 1$ near the critical transition region.

\paragraph{Models.}
We consider two representative families of main visual architectures. The first family consists of convolutional neural networks, for which we use ResNet18, ResNet34, and ResNet50~\citep{he2016deep} as models with three representative capacity levels. These models are used in the experiments in Sections~\ref{sec:control} and~\ref{sec:phase_transition}. The second family consists of Vision Transformers with global self-attention, for which we choose ViT-Tiny, ViT-Small, and ViT-Base, all with patch size $16$ and input resolution $224$~\citep{dosovitskiy2021image}. These models are used for the cross-architecture comparison in Section~\ref{sec:architecture}. All main models take single-channel grayscale images as input and produce binary classification outputs. Section~\ref{sec:architecture} also includes an additional pixel-token Transformer ablation, described there.

\paragraph{Training setup.}
Unless otherwise specified, the main ResNet and patch-based ViT models are initialized with ImageNet-pretrained weights~\citep{deng2009imagenet} and fine-tuned on our dataset. Therefore, these models have strong visual priors induced by pretraining before task-specific optimization begins. For these main experiments, all images are resized to $224 \times 224$ and normalized using ImageNet statistics. During training, we apply label-preserving geometric augmentations, including random rotations by multiples of $90^\circ$, random horizontal flips, and random vertical flips, to discourage the models from relying on a fixed orientation or gap location. Validation and test images are processed only with resizing and normalization. Each configuration $(\text{model}, N_{\mathrm{train}}, N)$ is trained for 50 epochs with five independent random seeds, ranging from 42 to 46. We optimize the models with AdamW using cross-entropy loss, a batch size of 128, and a weight decay of $10^{-2}$. The learning rate is decayed to $10^{-6}$ with a cosine schedule. The initial learning rate is $3 \times 10^{-4}$ for CNNs and $1 \times 10^{-4}$ for ViTs. For each seed, we select the checkpoint with the highest validation accuracy and evaluate it on the test set. We report the mean test accuracy, computed as the arithmetic mean over the five seeds. Max Test Acc denotes the highest test accuracy among the five seeds and reflects the best performance achieved across different initializations.

\paragraph{From artificial images to real scenarios.}
Although the images in this study are artificially generated, their visual features correspond to samples in real scenarios. A geometrically complete pattern can correspond to healthy organ tissue without lesions in medical imaging, or to a qualified industrial component without defects in industrial inspection. The capability boundary revealed in these experiments therefore has implications beyond the artificial dataset itself.

\begin{table}[!ht]
\centering
\caption{Average test accuracy on Task A, where $d_{\mathrm{syn}} > 0$, with $N_{\mathrm{train}} = 10{,}000$. Accuracy stays close to 1.0 at all scales and shows no clear decay.}
\label{tab:taskA}
\resizebox{\textwidth}{!}{%
\begin{tabular}{lccccccccccc}
\toprule
$N$ & 20 & 40 & 60 & 80 & 100 & 120 & 140 & 160 & 180 & 200 & 220 \\
\midrule
ResNet18 & 0.9996 & 0.9998 & 0.9998 & 0.9999 & 1.0000 & 1.0000 & 1.0000 & 0.9999 & 1.0000 & 1.0000 & 1.0000 \\
ResNet34 & 0.9999 & 1.0000 & 0.9999 & 0.9998 & 0.9999 & 1.0000 & 0.9999 & 0.9999 & 1.0000 & 1.0000 & 1.0000 \\
ResNet50 & 0.9998 & 0.9999 & 0.9996 & 0.9999 & 0.9999 & 1.0000 & 1.0000 & 1.0000 & 1.0000 & 1.0000 & 1.0000 \\
\bottomrule
\end{tabular}
}
\end{table}

\subsection{Control Experiment: Verifying Low-Level Shape Detection}
\label{sec:control}

\paragraph{Purpose.}
This experiment addresses a possible concern: the model may fail simply because the background noise is too strong for it to perceive even basic shapes.

\paragraph{Result.}
On Task A, which distinguishes pure random noise from closed squares and satisfies $d_{\mathrm{syn}} > 0$, all three ResNet architectures maintain classification performance close to 1.0 across the full scale range $N \in \{20, 40, \ldots, 220\}$. Table~\ref{tab:taskA} reports the average test accuracy of the three models at each scale when $N_{\mathrm{train}} = 10{,}000$. Accuracy is no lower than $99.96\%$ at all scales. Figure~\ref{fig:taskA} further shows the result curves of the three architectures. Accuracy remains close to 1.0 overall and does not show clear decay as the scale increases.

\begin{figure}[ht]
\centering
\begin{subfigure}[b]{0.82\textwidth}
\centering
\includegraphics[width=\linewidth]{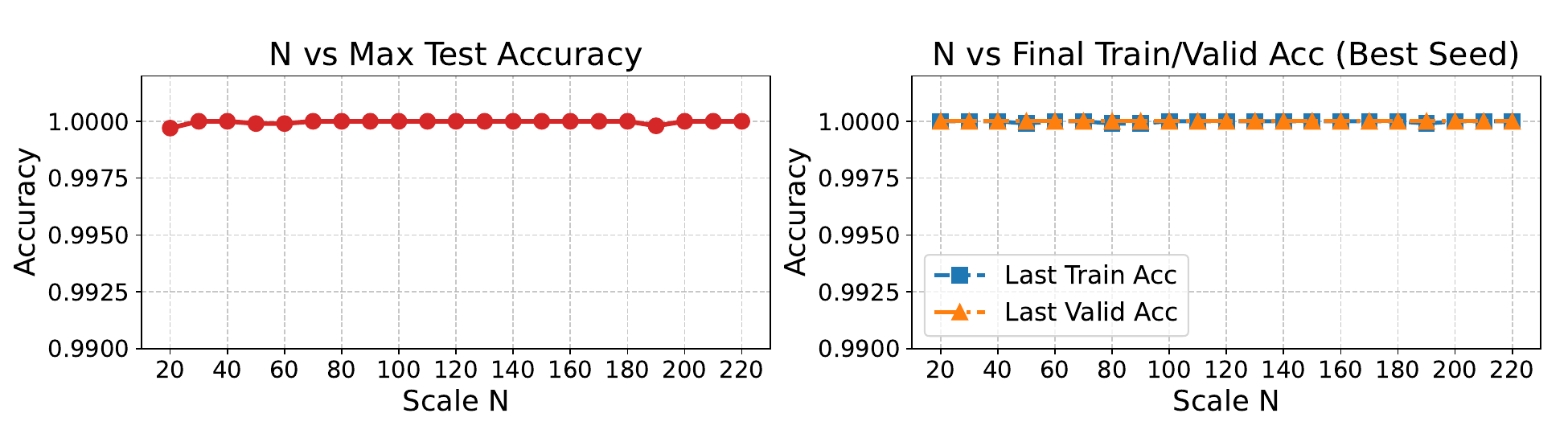}
\caption{Model: ResNet18, Task: A, N\_Train: 10000.}
\end{subfigure}
\begin{subfigure}[b]{0.82\textwidth}
\centering
\includegraphics[width=\linewidth]{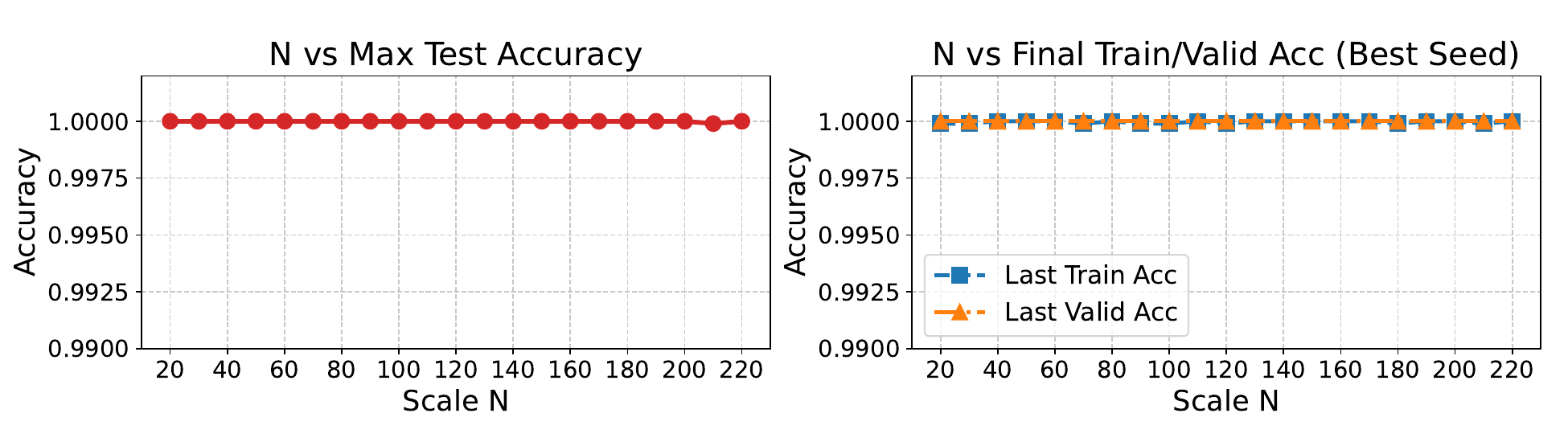}
\caption{Model: ResNet34, Task: A, N\_Train: 10000.}
\end{subfigure}
\begin{subfigure}[b]{0.82\textwidth}
\centering
\includegraphics[width=\linewidth]{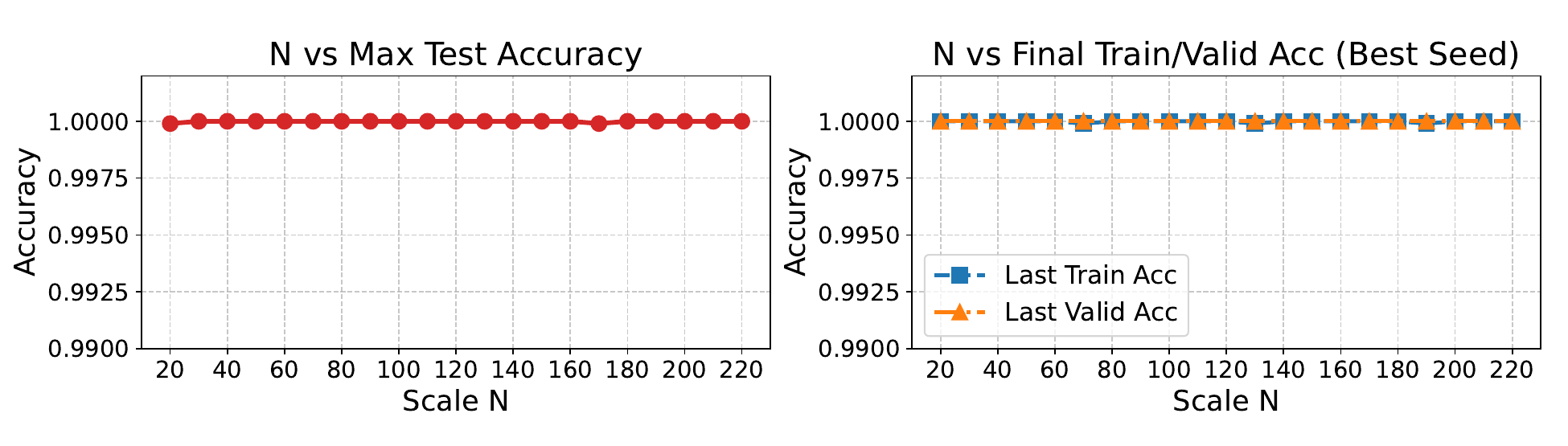}
\caption{Model: ResNet50, Task: A, N\_Train: 10000.}
\end{subfigure}
\caption{Best test accuracy and corresponding training and validation accuracy of three ResNet architectures on Task A, where $d_{\mathrm{syn}} > 0$ and $N_{\mathrm{train}} = 10{,}000$, as functions of scale $N$. The curves stay close to $100\%$, indicating that the high-variance noise background does not substantially interfere with classification on tasks with $d_{\mathrm{syn}} > 0$.}
\label{fig:taskA}
\end{figure}

\paragraph{Implication.}
These results show that deep learning models have stable low-level shape detection and local feature extraction abilities. When the classes contain clear statistical differences, that is, when $d_{\mathrm{syn}} > 0$, local feature fitting provides an effective classification basis, and the high-variance background does not substantially affect the model. Therefore, the later failure on Task B is hard to attribute to an inability to see the shape.

\subsection{Core Finding: Phase-Transition Collapse in Concept Recognition}
\label{sec:phase_transition}

\paragraph{Purpose.}
In this paper, we use \emph{concept recognition} in an operational sense. 
It refers to a model's ability to classify an image according to the global semantic predicate that defines the class, rather than according to local statistical cues or memorized pixel configurations. 
For Task B, this predicate is whether the square boundary is complete. 
Thus, stable concept recognition would require the model to recognize square completeness across increasing image scales under the constraint $d_{\mathrm{syn}}=0$.


\begin{figure}[!ht]
\centering
\includegraphics[width=\textwidth]{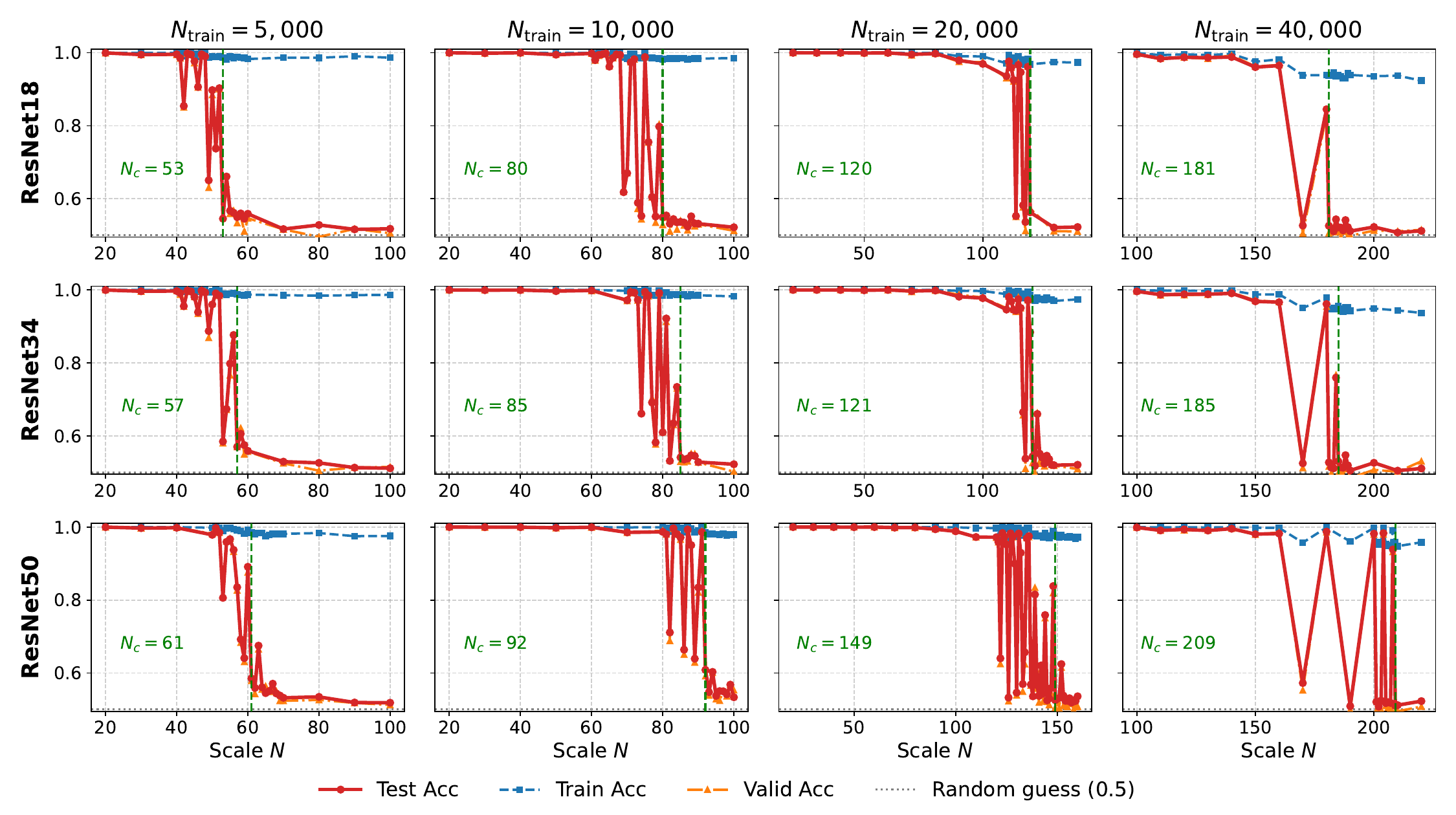}
\caption{Accuracy of three ResNet architectures, with rows corresponding to ResNet18, ResNet34, and ResNet50, on Task B with $d_{\mathrm{syn}} = 0$ as scale $N$ varies. The four columns correspond to $N_{\mathrm{train}} = 5{,}000 / 10{,}000 / 20{,}000 / 40{,}000$. Each panel overlays the test, training, and validation accuracy of the best seed. The gray dashed line marks random guessing at $0.5$, and the green vertical dashed line marks the critical collapse point $N_c$. The three-stage phase transition, from high accuracy to fluctuation and then to random guessing, appears consistently across all 12 configurations. Training accuracy stays close to $1.0$ for a long range, indicating that training-set memorization persists, while test and validation accuracy collapse to around $0.5$ after crossing $N_c$. The critical point $N_c$ shifts markedly to the right as training size and model capacity increase, but all settings enter the random-guessing regime at larger scales within the present experimental range.}
\label{fig:taskB_resnet}
\end{figure}

\paragraph{Result: a clear three-stage phase transition.}
As shown in Figure~\ref{fig:taskB_resnet}, all three ResNet architectures exhibit highly consistent three-stage behavior on Task B. In the high-accuracy region, where $N < N_c$, test accuracy remains above $99\%$. Near the collapse point, there is an unstable transition region in which accuracy fluctuates non-monotonically between $0.5$ and $0.99$. In the random-guessing collapse region, where $N \ge N_c$, performance degrades to random guessing, and accuracy remains in the interval $[0.50, 0.55]$ for a long range.

\begin{table}[!ht]
\centering
\small
\caption{Collapse critical point $N_c$ on Task B under different architectures and training sizes. The critical point shifts to larger values as training size and model capacity increase, but all settings still fall to nearly 50\% accuracy at larger $N$.}
\label{tab:Nc}
\begin{tabular}{lcccc}
\toprule
Model & $N_{\mathrm{train}}=5{,}000$ & $N_{\mathrm{train}}=10{,}000$ & $N_{\mathrm{train}}=20{,}000$ & $N_{\mathrm{train}}=40{,}000$ \\
\midrule
ResNet18 & 53 & 80 & 120 & 181 \\
ResNet34 & 57 & 85 & 121 & 185 \\
ResNet50 & 61 & 92 & 149 & 209 \\
\bottomrule
\end{tabular}
\end{table}

\paragraph{Critical point $N_c$.}
We define the critical scale $N_c$ as the smallest $N$ such that, for all $N' \ge N_c$, the highest test accuracy among the five seeds, namely Max Test Acc, is below $0.7$. This definition uses the best seed as the criterion. As long as any seed still exceeds $0.7$, the model is not counted as collapsed. It also requires all subsequent scales to stay below $0.7$, which places the critical point strictly at the end of the fluctuating transition band and at the start of the random-guessing collapse region. This avoids misidentifying occasional high spikes in the transition region as recovery. Table~\ref{tab:Nc} summarizes the critical points for ResNet18, ResNet34, and ResNet50 under four training sizes $N_{\mathrm{train}} \in \{5{,}000,\ 10{,}000,\ 20{,}000,\ 40{,}000\}$.

\paragraph{Finding 1: more data delays collapse but does not yet prevent it.}
The horizontal changes across $N_{\mathrm{train}}$ in Table~\ref{tab:Nc} show that increasing the training size usually shifts the critical point $N_c$ to the right. For ResNet50, as $N_{\mathrm{train}}$ increases from $5{,}000$ to $40{,}000$, an eightfold increase, $N_c$ moves from $61$ to $209$, about a 3.4-fold increase. Thus, the collapse point grows much more slowly than the training size itself. However, under all four training sizes, accuracy falls to nearly 50\% within the observed scale range. The main difference is where the collapse occurs. This closely matches the prediction from $d_{\mathrm{syn}} = 0$: increasing the training set mainly allows the model to cover a larger solution space and thereby delay the phase transition, but it is not equivalent to learning the global concept of whether the square is complete. Once the solution space exceeds the model's finite memorization capacity, the classification strategy no longer generalizes. In other words, the model appears to maintain performance within a finite scale range by relying on broader training coverage. Yet sample coverage is constrained by finite data, while the concept of square completeness is not constrained by finite sample size.

\paragraph{Finding 2: model capacity also mainly delays collapse.}
Comparing the three rows in Table~\ref{tab:Nc} shows that, at the same $N_{\mathrm{train}}$, the larger ResNet50 obtains only a limited shift in $N_c$ relative to ResNet18. When $N_{\mathrm{train}} = 40{,}000$, the shift is from $181$ to $209$, about $15\%$. There is still no sign that the larger model avoids entering the random-guessing regime at larger scales. Thus, model capacity and training size play similar roles on Task B. Both appear to enlarge the number of patterns that can be memorized within the existing feature language, but neither provides evidence of breaking the capability boundary set by $d_{\mathrm{syn}} = 0$.

\paragraph{Implication.}
These results constitute the key empirical finding of this paper: a phase transition. When the scale $N$ is small, the model can use its high parameter capacity to memorize local compositions in a finite solution space. As $N$ increases, however, $d_{\mathrm{syn}} = 0$ explains why local fitting strategies fail to generalize. The path of approximating a global concept with finite feature rules quickly becomes unstable. The additional observation that both data size and model capacity mainly delay but do not eliminate collapse supports the robustness of this conclusion along multiple dimensions. The phase transition is not an accident of a particular architecture or a particular data size. It is more likely a structural consequence of tasks with $d_{\mathrm{syn}} = 0$.

\subsection{Architectural Limitations}
\label{sec:architecture}

\begin{table}[ht]
\centering
\small
\caption{Collapse critical point $N_c$ of Vision Transformers on Task B, where $d_{\mathrm{syn}} = 0$. Increasing the training size still delays collapse, but at the same training size, the $N_c$ of ViTs is much smaller than that of CNNs, as compared with Table~\ref{tab:Nc}. The largest model, ViT-Base, does not obtain a later collapse point.}
\label{tab:Nc_vit}
\begin{tabular}{lcccc}
\toprule
Model & $N_{\mathrm{train}}=5{,}000$ & $N_{\mathrm{train}}=10{,}000$ & $N_{\mathrm{train}}=20{,}000$ & $N_{\mathrm{train}}=40{,}000$ \\
\midrule
ViT-Tiny  & 37 & 50 & 65 & 89 \\
ViT-Small & 40 & 61 & 68 & 89 \\
ViT-Base  & 44 & 50 & 69 & 68 \\
\bottomrule
\end{tabular}
\end{table}

\paragraph{Purpose.}
The phase-transition collapse in Section~\ref{sec:phase_transition} is observed on CNNs. A natural question is whether the collapse comes from the local receptive fields inherent to convolution. Vision Transformers provide a useful comparison because they replace convolutional receptive fields with patch-wise global self-attention and exhibit different inductive biases from CNNs~\citep{dosovitskiy2021image,naseer2021intriguing}. To test whether this architectural change can cross the boundary set by $d_{\mathrm{syn}} = 0$, we apply the same scale range and training setup to three Vision Transformer capacities, ViT-Tiny, ViT-Small, and ViT-Base, on Task B. This examines whether phase-transition collapse is a structural limitation of the current deep learning paradigm or merely a special case of one network architecture.

\paragraph{Result: global attention does not prevent collapse and collapses earlier.}
Table~\ref{tab:Nc_vit} summarizes the collapse critical points $N_c$ for the three ViTs under four training sizes, and Figure~\ref{fig:taskB_vit} shows their full phase-transition curves. The results show three clear facts. First, ViTs display a three-stage transition similar to CNNs: accuracy is close to $100\%$ at small scales, collapses to the random-guessing interval $[0.50, 0.55]$ after the critical point, and does not recover within the observed range. Global attention does not produce a qualitative change in this setting. Second, collapse occurs much earlier than in CNNs. At the same training size, the critical point $N_c$ of ViTs is about half, or even less, of the corresponding CNN value. For example, when $N_{\mathrm{train}} = 40{,}000$, the $N_c$ values of the ResNet family lie between $181$ and $209$ as shown in Table~\ref{tab:Nc}, whereas all three ViTs have already collapsed by $N_c \le 89$. The global-attention architecture, which has weaker spatial priors, shows no advantage on this task and instead enters the random-guessing regime earlier. Third, increased capacity does not yield concept learning. The monotonic trend seen in CNNs, where larger capacity delays collapse, does not hold for ViTs. The collapse points of the three ViTs are close to each other, and the largest model, ViT-Base, does not collapse systematically later. Under $N_{\mathrm{train}} = 40{,}000$, it even collapses earliest, with $N_c = 68$, while ViT-Tiny and ViT-Small both have $N_c = 89$. Increasing the scale of the attention model does not show evidence that the model learns the global concept of whether the square is complete.

Consistent with the CNN results, increasing the training size still delays ViT collapse. For example, the $N_c$ of ViT-Tiny increases from $37$ to $89$ as $N_{\mathrm{train}}$ grows. Yet this effect still mainly delays rather than eliminates collapse. Within the current experimental range, all settings fall into the 50\% random region at larger scales.

\begin{figure}[ht]
\centering
\includegraphics[width=\textwidth]{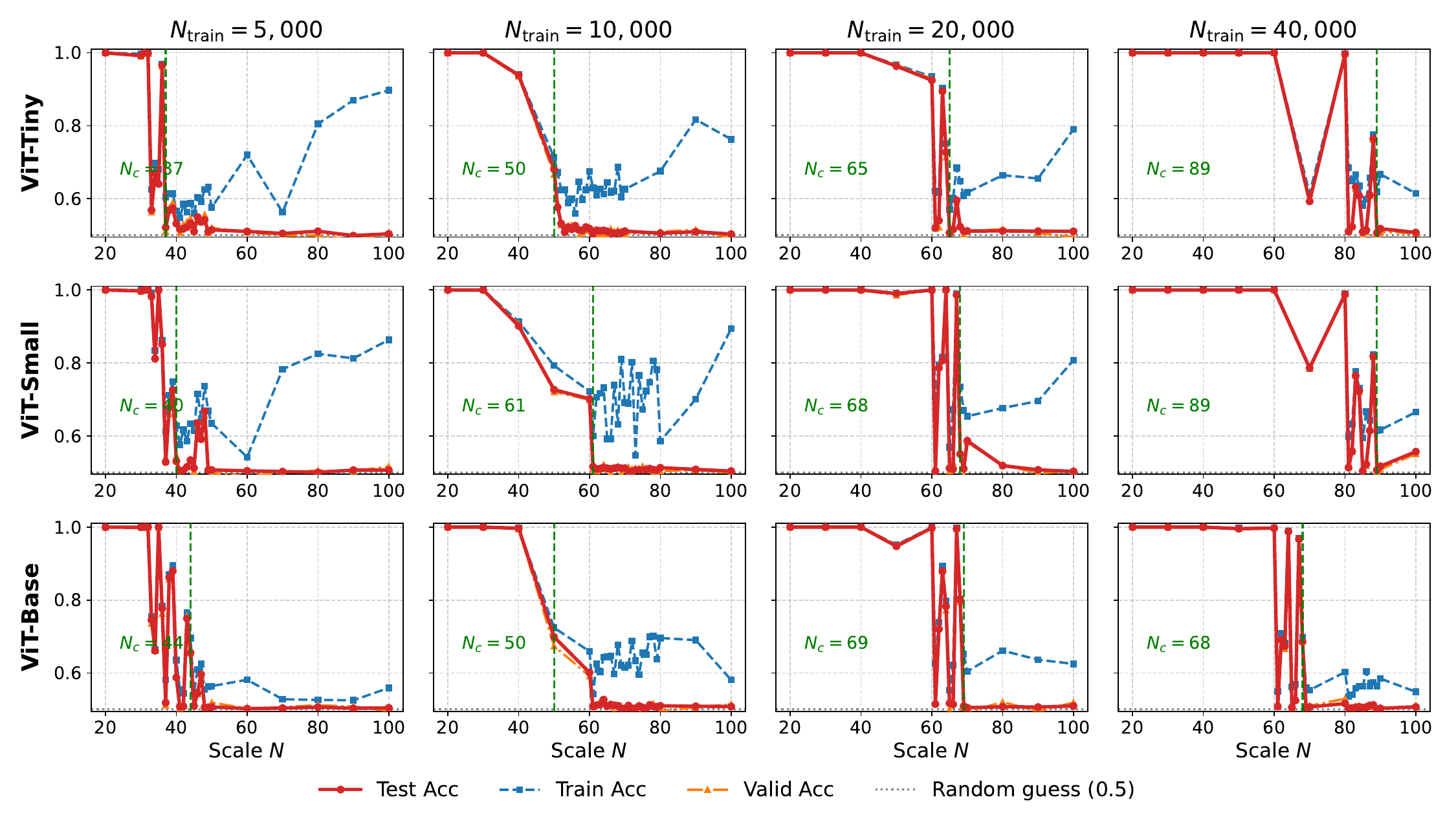}
\caption{Accuracy of three ViT architectures, with rows corresponding to ViT-Tiny, ViT-Small, and ViT-Base, on Task B with $d_{\mathrm{syn}} = 0$ as scale $N$ varies. The four columns correspond to $N_{\mathrm{train}} = 5{,}000 / 10{,}000 / 20{,}000 / 40{,}000$. The legend and $N_c$ annotation follow Figure~\ref{fig:taskB_resnet}. The three-stage phase transition, from high accuracy to fluctuation and then to random guessing, resembles the CNN case, but the critical point $N_c$ appears much earlier. All ViTs collapse at $N_c \le 89$, whereas CNNs under the same training size range between $181$ and $209$. This indicates that global attention does not yet show evidence of breaking the capability boundary set by $d_{\mathrm{syn}} = 0$ in this setting.}
\label{fig:taskB_vit}
\end{figure}

\paragraph{Pixel-token Transformer without patch aggregation.}
To isolate patch-level aggregation as a possible source of the ViT collapse, we further remove the standard patch embedding and train a Transformer from scratch on pixel tokens. Unlike the main ResNet and patch-based ViT experiments, this model takes the generated $N \times N$ image directly, without resizing it to $224 \times 224$. Each original pixel is treated as one token. The model therefore uses neither $16\times16$ patch aggregation nor ImageNet pretraining. The pixel-token Transformer uses embedding dimension $128$, depth $4$, and $4$ attention heads. Its optimizer, batch size, weight decay, seed protocol, and cosine learning-rate schedule follow the ViT setting, with the initial learning rate set to $1 \times 10^{-4}$ and decayed to $10^{-6}$. Although this configuration has fewer parameters than ViT-Tiny, its pixel-level sequence length makes the comparison not compute-matched. Around its collapse point, its block-level compute lies between that of ViT-Tiny and ViT-Small. We evaluate it on Task B with $N_{\mathrm{train}} = 10{,}000$ using a dense scale sweep around the collapse region. As shown in Figure~\ref{fig:taskB_pixel}, it exhibits the same three-stage phase transition, but collapses earlier than the patch-based ViTs: its critical point is $N_c = 34$, compared with $N_c$ between $50$ and $61$ for the patch-based ViTs at the same training size. The transition band is also markedly noisier, with different seeds at the same scale ranging from random guessing to near-perfect accuracy. Thus, replacing patch tokens with pixel tokens and removing pretraining does not help cross the $d_{\mathrm{syn}} = 0$ boundary. Instead, collapse occurs earlier.

\begin{figure}[ht]
\centering
\includegraphics[width=0.72\textwidth]{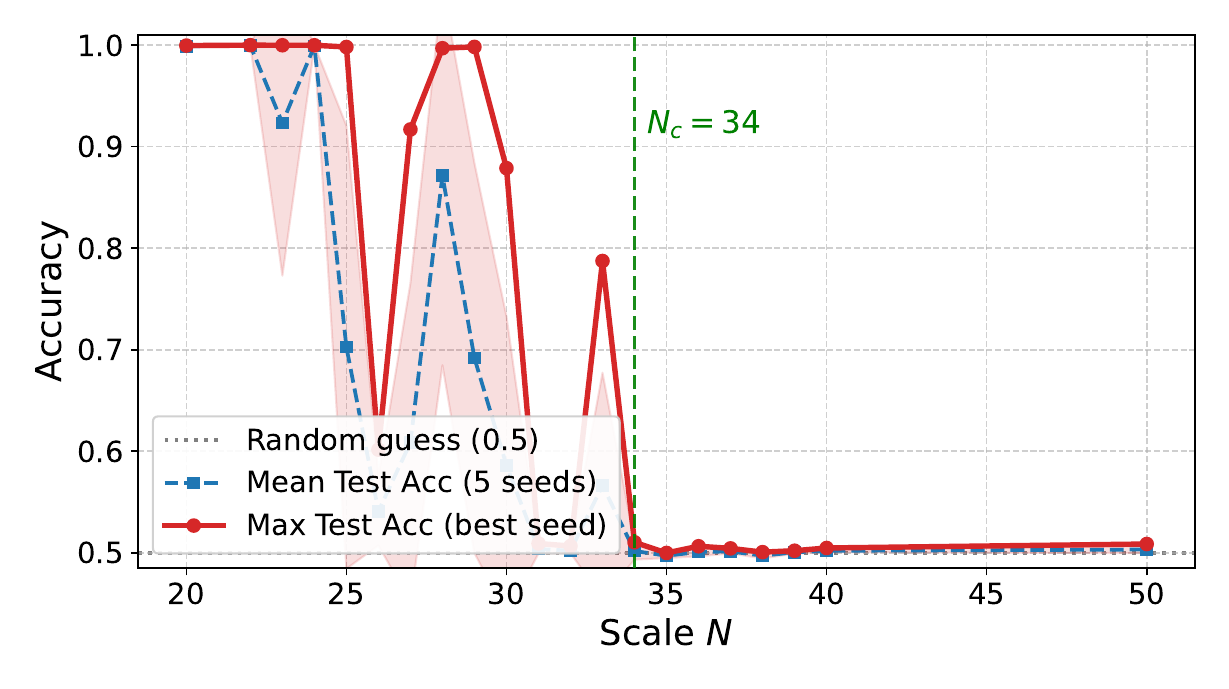}
\caption{Pixel-token Transformer on Task B ($d_{\mathrm{syn}} = 0$, $N_{\mathrm{train}} = 10{,}000$): each original pixel of the generated $N \times N$ image is treated as one token, with no resizing to $224 \times 224$, no $16\times16$ patch aggregation, and no ImageNet pretraining. The model is trained from scratch with embedding dimension $128$, depth $4$, and $4$ attention heads. The red curve is the best-seed (Max) test accuracy and the blue dashed curve is the 5-seed mean. The shaded band shows $\pm 1$ standard deviation, the gray dotted line marks random guessing at $0.5$, and the green dashed line marks the collapse point $N_c = 34$. The three-stage transition from high accuracy to fluctuation and then to random guessing is preserved, but the collapse occurs earlier than for the patch-based ViTs at the same training size ($N_c$ between $50$ and $61$), and the transition band is visibly noisier.}
\label{fig:taskB_pixel}
\end{figure}

\paragraph{Implication.}
Within this ViT section, the patch-based ViTs and the pixel-token Transformer test two attention-based ways of representing the same $d_{\mathrm{syn}} = 0$ task: standard patch tokens with pretrained visual features, and original-pixel tokens trained from scratch. Neither setting enables the model to create a new rule for capturing the global completeness concept. Because the pixel-token model differs in parameter count, sequence length, and pretraining, its earlier collapse should not be read as a capacity-matched comparison against ViT-Tiny, ViT-Small, or ViT-Base. The narrower implication is that the failure is not simply caused by the $16\times16$ patch embedding or by ImageNet-pretrained ViT features: removing both does not make the task more conceptually accessible. Attention changes how information is aggregated, but it does not by itself supply a new descriptive language for the global concept required by the task.

\section{Discussion}
\label{sec:discussion}

The phase-transition collapse observed in our experiments can be understood from the perspective of compression. Intelligence can be viewed as compressing sensory input into representations that preserve the distinctions needed for reasoning and action. Current deep learning models also perform a form of compression: within an existing feature language, such as convolution kernels, attention patterns, and learned feature representations, they compress statistical regularities in the training data into model parameters. This compression can be effective when $d_{\mathrm{syn}} > 0$, because local features already carry discriminative information.

The situation changes when $d_{\mathrm{syn}} = 0$. In our visual self-referential task, local feature descriptions provide no stable basis for the distinction between closed and one-pixel-broken squares. At small scales, models can still maintain high accuracy by fitting a finite sample space, but as the scale grows, this strategy becomes increasingly unstable and eventually collapses to random guessing. This suggests that stable performance on such tasks requires a different kind of compression: forming a new descriptive concept that captures global semantics, rather than fitting more instances within an existing feature language~\citep{genewein2026agi}.

This perspective also connects to two broader directions in current artificial intelligence: world models and embodied AI. World models aim to learn internal representations of the environment that support prediction, planning, and decision-making~\citep{ha2018world,lecun2022path}. Although this objective appears to target physical structure, many existing world models still rely heavily on statistical regularities in observed data and have not fully captured the causal structure and physical invariances of the environment~\citep{ding2024understanding}. They can interpolate efficiently within familiar distributions, but they remain limited when faced with situations that require forming new concepts for unseen physical phenomena.

Embodied AI approaches intelligence through direct interaction between an agent and the physical world. This direction is closely related to the symbol grounding problem~\citep{harnad1990symbol}: concepts acquire stable meaning only when they are grounded in sensorimotor experience, rather than defined only through other symbols. Yet many recent embodied systems still inherit much of their conceptual structure from large language models~\citep{driess2023palme,brohan2023rt2}. In such systems, language models often provide concepts and planning, while the robot performs low-level execution. This can enable strong functional performance, but the conceptual system is largely borrowed rather than autonomously formed through physical interaction.

Thus, the limitation observed in our visual self-referential task may reflect a broader bottleneck shared by several current AI paradigms. Vision models fit feature patterns in pixel space, world models predict statistical structure in state space, and embodied systems often inherit concepts from pretrained language models. These routes differ in implementation, but all raise the same question: can an intelligent system form new descriptive concepts when its existing language is insufficient~\citep{genewein2026agi}?

\section{Conclusion}
\label{sec:conclusion}

Returning to Wittgenstein's view that the limits of language are the limits of the world, this paper treats visual recognition as bounded by the descriptive language through which an object can be represented and classified. For current vision models, this language is largely instantiated by learned feature representations. The central question is therefore whether a model can still recognize a global object concept when its local feature language provides no stable basis for distinction.

We formalized this question with syntactic distance and constructed visual self-referential instances in which closed and one-pixel-broken squares differ in global semantics but have $d_{\mathrm{syn}} = 0$. Across ResNet and Vision Transformer architectures, accuracy exhibited a consistent phase-transition collapse: it remained high at small image scales, fell to the random-guessing regime after a critical scale, and did not recover within the observed range. Larger data and models delayed but did not remove this boundary. These results suggest that current architectures still operate mainly within an existing descriptive language. When recognition requires forming a new concept for global semantics, their capability boundary becomes visible. Therefore, the path toward AGI may lie not merely in enabling machines to master more languages, but also in equipping them to create language.

\bibliographystyle{plainnat}
\bibliography{references}

\end{document}